\documentclass[12pt, final]{l4dc2022}

\title[Online Learning of Control Barrier Conditions for Safety-Critical Control]{Barrier Bayesian Linear Regression: Online Learning of Control Barrier Conditions  for Safety-Critical Control of Uncertain Systems}
\usepackage{times}

\author{%
 \Name{Lukas Brunke}$^{1, 3, 4*}$ \Email{lukas.brunke@robotics.utias.utoronto.ca}\\[-1em]
 \AND
 \Name{Siqi Zhou}$^{1, 3, 4*}$
 \Email{siqi.zhou@robotics.utias.utoronto.ca}\\[-1em]
 \AND
 \Name{Angela P. Schoellig}$^{1, 2, 3, 4}$ \Email{angela.schoellig@utoronto.ca}\\[1em]
$^{1}$\addr University of Toronto Institute for Aerospace Studies, Toronto, Ontario, Canada\\
$^{2}$\addr Technical University of Munich, Munich, Germany\\
$^{3}$\addr University of Toronto Robotics Institute, Toronto, Ontario, Canada\\
$^{4}$\addr Vector Institute for Artificial Intelligence, Toronto, Ontario, Canada
\\
$^{*}$Equal contribution%
}

\usepackage{amssymb}
\usepackage{amsmath}
\usepackage{siunitx}
\usepackage{graphicx}
\usepackage{tikz}
\usepackage{pgfplots}
\pgfplotsset{compat=1.16}

\usepackage{booktabs}
\usepackage{multirow}
\usepackage{color, colortbl}
\usepackage{hyperref}

\allowdisplaybreaks

\usetikzlibrary{plotmarks}
\usetikzlibrary{arrows.meta}
\usepgfplotslibrary{patchplots}
\usepackage{grffile}

\renewcommand\vec{\mathbf} 
\newcommand{\R}{\mathbb{R}}

\newcommand{\N}{\mathbb{N}}

\newcommand{\f}{\vec{f}}
\newcommand{\fhat}{\hat{\f}}
\newcommand{\g}{\vec{g}}
\newcommand{\ghat}{\hat{\g}}

\newcommand{\x}{\vec{x}}
\newcommand{\xdot}{\dot{\x}}
\newcommand{\xhat}{\hat{\x}}
\newcommand{\xhatdot}{\dot{\xhat}}

\renewcommand{\u}{\vec{u}}

\begin{document}

\maketitle

\begin{abstract}%
In this work, we consider the problem of designing a safety filter for a nonlinear uncertain control system. Our goal is to augment an arbitrary controller with a safety filter such that the overall closed-loop system is guaranteed to stay within a given state constraint set, referred to as being safe. For systems with known dynamics, control barrier functions~(CBFs) provide a scalar condition for determining if a system is safe. For uncertain systems, robust or adaptive CBF certification approaches have been proposed. However, these approaches can be conservative or require the system to have a particular parametric structure. For more generic uncertain systems,   machine learning approaches have been used to approximate the CBF condition. These works typically assume that the learning module is sufficiently trained prior to deployment. Safety during learning is not guaranteed.  We propose a barrier Bayesian linear regression~(BBLR) approach that  guarantees safe online learning of the CBF condition for the true, uncertain system. We assume that the error between the  nominal system and the true system is bounded and exploit the structure of the CBF condition. We show that our approach can safely expand the set of certifiable control inputs despite system and learning uncertainties. The effectiveness of our approach is demonstrated in simulation using a two-dimensional pendulum stabilization task. 
\end{abstract}

\vspace{1em}

\section{Introduction}
Robots are envisioned to perform increasingly complex tasks in applications ranging from autonomous driving to space exploration. In these applications, robots are required to cope with uncertainties in the operating environment, and their safe operation  is critical. In the literature, nonlinear control techniques, such as Lyapunov stability~\citep{khalil2002}, Hamilton-Jacobi reachability analysis~\citep{mitchell2005time}, or control barrier functions (CBFs)~\citep{ames2019a}, provide the mathematical foundation for deriving safety conditions when the system dynamics are known. However, it remains challenging to provide safety guarantees if the robot dynamics are uncertain~\citep{DSL2021}. In this work, we consider the problem of keeping the system inside of a safe state constraint set. We leverage the CBF certification framework and Bayesian linear regression (BLR) to design an adaptive safety filter that renders any given controller safe despite uncertainties present in the system (\autoref{fig:blockdiagram}). The proposed approach exploits the structure of the CBF condition and ensures that an online learned CBF condition is valid by construction.

Inspired by Nagumo's theorem~\citep{Nagumo1942berDL}, the CBF certification framework provides necessary and sufficient safety conditions for a nonlinear control system to stay within a predefined state constraint set~\citep{ames2019a}. In~\citep{Ames2014}, CBFs have been used as safety filters to modify unsafe control inputs computed by an arbitrary controller. However, this framework relies on knowing the dynamics of the system. Any uncertainty in the dynamics results in uncertainty in the CBF certification condition, which can lead to violations of the safety constraints. One possibility to account for these uncertainties is by employing robust~\citep{choi2021robust} or adaptive~\citep{taylor2020c,Lopez2020} CBF conditions to recover safety. While these approaches guarantee safety, robust approaches are typically overly conservative and adaptive approaches often require an \textit{a priori} known structure of the uncertain dynamics, which can be restrictive in practice. 

\begin{figure}[tb]
    \centering
    \includegraphics[scale=.45]{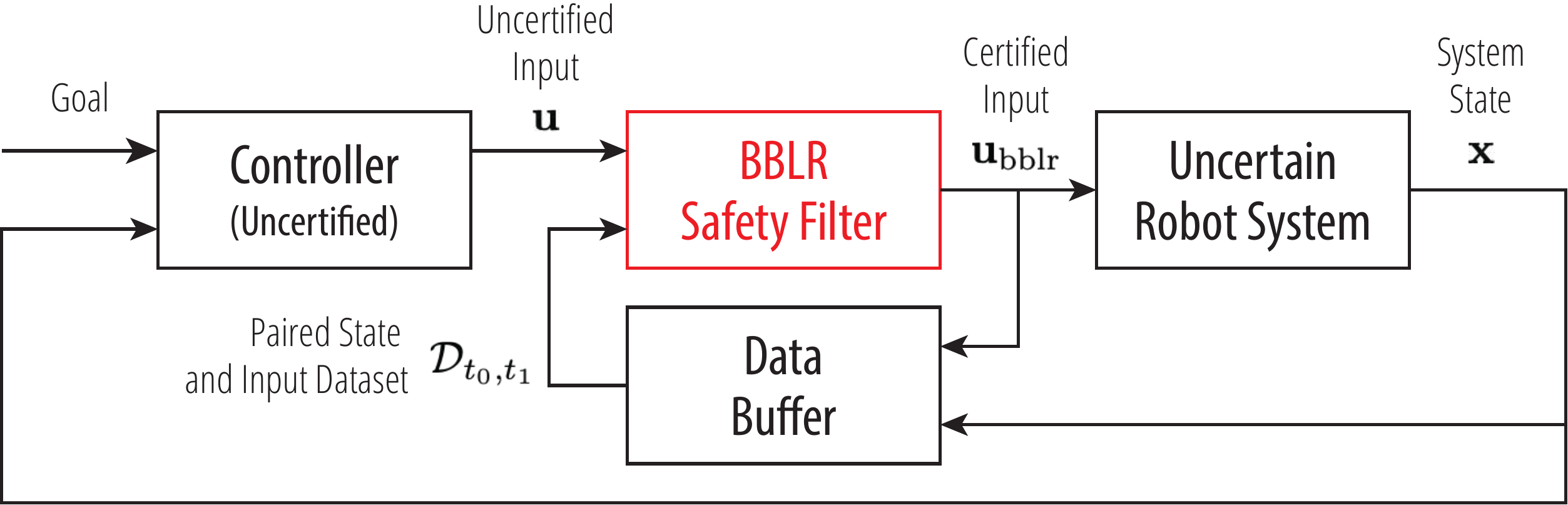}
    \caption
    {
        A block diagram of the proposed barrier Bayesian linear regression (BBLR) safety filter approach, which aims to safely learn an uncertain CBF condition online. Augmenting a given, and possibly unsafe controller, with the BBLR safety filter (in red) guarantees that the state of the uncertain robot system stays within a state constraint set. 
    }
    \label{fig:blockdiagram}
\end{figure}

Recently, learning-based approaches have been developed to address these drawbacks and learn the CBF from data. Compared to  approaches that try to  learn the system dynamics~\citep{DSL2021}, the advantage of CBF approaches is that they  directly learn the error in the CBF condition, which is typically of lower dimension than the full system dynamics~\citep{taylor2019a}. In~\citep{Choi2020, taylor2020a, taylor2020b}, neural network models are used to learn the  CBF conditions offline from pre-collected data. By representing the CBF condition error by a generic neural network, a larger class of uncertain systems can be considered compared to other adaptive approaches. In~\citep{Wang2018a, Ohnishi2019, cheng2019a, Fan2019, khojasteh2020a}, probabilistic learning techniques for CBFs have been proposed to  explicitly account for the uncertainty of the learned CBF condition and guarantee safety of the system probabilistically.

However, current learning-based approaches do not guarantee safety while learning the CBF condition, or only do so by assumption~\citep{taylor2020b}. This is due to the fact that the learned approximations typically do not have a lower bound that can be conveniently incorporated into the CBF safety analysis.
Furthermore, the data collected for training the CBF condition may be subject to errors (e.g., due to numerical differentiation or noise in the system)~\citep{khojasteh2020a} or may be insufficient for  learning the CBF condition well~\citep{taylor2020a, taylor2020b}.
These errors in the learned CBF condition are not explicitly accounted for in current approaches and can result in unsafe behavior.

In this work, we present an online learning framework, called barrier Bayesian linear regression (BBLR), that exploits the structure of the CBF condition and provably guarantees the satisfaction of the CBF condition \textit{during learning}. 
Under the assumption that the error between a given nominal system and the true system is bounded, we prove that our proposed BBLR algorithm can expand the set of safe control inputs while keeping the system  inside of the safe set. Additionally, our Bayesian approach captures the uncertainty of the learned CBF condition, which allows us to incorporate the uncertainty of the learned CBF condition in the overall safety filter framework. We illustrate the improved performance and safety of our approach on a simulated two-dimensional pendulum.

\section{Problem Formulation}
We consider the control architecture in~\autoref{fig:blockdiagram}. The uncertain robot system is modeled as a continuous-time control-affine system with perfect state measurements: 
\begin{equation}
\label{eq:nonlinear_affine_control}
	\xdot = \f(\x) + \g (\x) \u\,,
\end{equation}
where $\x\in \set{X} \subset \R^n$ is the state of the system, $\u \in \mathcal{U} \subset \R^m$ is the control input and $\set{U}$ is compact, and $\f : \R^n \to \R^{n}$ and $\g : \R^n \to \R^{n \times m}$ are locally Lipschitz continuous functions. 
We refer to this system as the true system which is partially unknown. 

We consider a nominal control-affine system given by
\begin{equation}
\label{eq:nominal_system}
	\xhatdot = \fhat(\x) + \ghat (\x) \u \,,
\end{equation}
where $\fhat : \R^n \to \R^{n}$ and $\ghat : \R^n \to \R^{n \times m}$ are known, locally Lipschitz continuous functions. 
Without loss of generality, we can write the true dynamics as
\begin{equation}
    \label{eq:disturbed_system}
        \xdot = \fhat(\x) + \ghat (\x) \u + \vec{d}(\x, \u) \,,
\end{equation}
where $\vec{d}(\x, \u) = \vec{a}(\x) + \vec{b}(\x)^\intercal \u$ with $\vec{a} (\x) =  \f(\x) - \fhat(\x)$ and $\vec{b} (\x) = \g (\x) - \ghat (\x)$. We assume that the difference between the true and nominal system $\vec{d}(\x, \u)$ is bounded for all $(\vec{x},\vec{u})\in \set{X}\times \set{U}$.

Our goal is to augment a given, not certified state-feedback controller 
with a safety filter such that the true system's state $\x$ in~\autoref{eq:nonlinear_affine_control} stays inside a safe set $\set{C}$, see~\autoref{fig:blockdiagram}. The safe set  $\set{C}$ is assumed to be compact and is defined as the $0$-superlevel set of a $C^1$ function $h: \R^n \to \R$:
\begin{equation}
\label{eq:set_c}
    \set{C} = \{\x \in \set{X} \subset \R^n \: \vert \: h(\x) \geq 0 \} \,.
\end{equation}
Its boundary is defined as $\partial\set{C} = \{\x \in \set{X} \subset \R^n \: \vert \: h(\x) = 0 \}$, where for all $\x \in \partial \set{C}$ it holds that $\partial h(\x) / \partial \x  \neq 0$. 
We assume that there exists a control input signal~$\nu: \mathbb{R}_{\ge 0}\mapsto \mathcal{U}$ such that the true system in~\autoref{eq:nonlinear_affine_control} stays inside $\set{C}$. 
We also assume that there exists a control input signal~$\nu$, such that the nominal system in~\autoref{eq:nominal_system} stays inside a subset $\set{C}_{\bar{h}}$ of the safe set $\set{C}$. 

\section{Technical Background}
In this section, we introduce the background for safely learning control barrier function conditions. We begin this section with the definitions of the extended class $\mathcal{K}_{\infty}$ comparison functions and positively control invariant sets. We then give an overview of control barrier functions and Bayesian linear regression.

\subsection{Definitions}
Comparison functions are used to analyze a nonlinear system's stability with Lyapunov functions. CBF conditions rely on comparison functions from the extended class $\set{K}_{\infty}$: 
\begin{definition}[Extended class $\mathcal{K}_\infty$ function]
\label{def:kappa_inf_extended}
	A continuous function $\gamma : \left(-b, a\right) \to \R$, with $\gamma(0) = 0$, $\gamma$ is strictly monotonically increasing, and $a, b = \infty$, $\lim_{r \to \infty} \gamma(r) = \infty$ and $\lim_{r \to - \infty} \gamma(r) = - \infty$, then $\gamma$ is said to belong to $extended~class~\mathcal{K}_\infty$ or equivalently expressed as $\gamma \in \mathcal{K}_{\infty, e}$. 
\end{definition}
In this work, safety and positively control invariance are used synonymously, although, other definitions exist~\citep{DSL2021}. We recall the definition of positively control invariance: 
\begin{definition}[Positively control invariant set] Let $\mathfrak{U}$ be the set of all $\nu : \R_{\geq 0} \to \set{U}\,$. A set ${\set{C}\subset\mathcal{X}}$ is a positively control invariant set for a control system $\xdot = \f(\x, \u)$ with $\f: \mathcal{X} \times \mathcal{U} \to \mathcal{X}$ if~$\: \forall \: {\x_0 \in\set{C}} \,,\: \exists \: \nu \in \mathfrak{U} \,, \: \forall \: t \in T_{\x_0}^+ \,,\: \phi(t, \x_0, \nu) \in \set{C}$, where $\phi(t, \x_0, \nu)$ is the phase flow of $\f(\x,\u)$ starting at $\x_0$ and following $\nu$, and $T_{\x_0}^+$ is the maximum time interval.
\end{definition}

\subsection{Control Barrier Functions}

We aim to find a controller $\vec{k} : \R^n \to \R^m$ for a control-affine system as in~\autoref{eq:nonlinear_affine_control} that renders the set $\set{C}$ positively control invariant. Suppose that $\u = \vec{k}(\x)$ is a state-feedback controller rendering the closed-loop system $\xdot = \f_{\mathrm{cl}}(\x) = \f(\x) + \g(\x) \vec{k}(\x)$
locally Lipschitz continuous on the domain $\set{X}$. Then, for any initial condition $\x_0 \in \set{X}$ there exists a maximum interval of existence~$T_{\x_0}^+$ such that $\x(t)$ is the unique solution to
the closed-loop system 
on~$T_{\x_0}^+$.  
The definition of a control barrier function following \citep{ames2019a} is given next. 
\begin{definition}[Control barrier function (CBF)]
	Let $\set{C} \subseteq \set{X} \subset \R^n$ be the superlevel set of a $C^1$ function $h: \mathcal{D} \to \R$, then $h$ is a control barrier function (CBF) if there exists an extended class $\set{K}_\infty$ function $\gamma$ such that for all $\x \in \set{X}$ the control system in~\autoref{eq:nonlinear_affine_control} satisfies
	\begin{equation}
		\label{eq:cbf_lie_derivative}
		\sup_{\u \in \mathcal{U}} \left[L_\vec{f} h(\x) + L_\vec{g} h(\x) \u \right] \geq - \gamma(h(x)) \,.
	\end{equation}
\end{definition}

Let $\set{U}_{\mathrm{CBF}}(\x) = \{ \u \in \mathcal{U} \: \vert \: L_\vec{f} h(\x) + L_\vec{g} h(\x) \u + \gamma(h(\x)) \geq 0 \}$ be the set of inputs satisfying the CBF condition.
Then, for any $\x \in \set{X}$ any Lipschitz continuous controller $\vec{k}(\x) \in \set{U}_{\mathrm{CBF}}(\x)$ guarantees positive invariance of $\set{C}$ under the dynamics in~\autoref{eq:nonlinear_affine_control}. Determining the positive invariance of a system amounts to checking if the Lie derivative condition in~\autoref{eq:cbf_lie_derivative} is satisfied. This is stated in the following theorem~\citep{Ames2017}: 

\begin{theorem}[CBF positive invariant set certification]
	\label{th:cbf_pi}
	Let $\set{C} \subset \R^n$ be defined as in~\autoref{eq:set_c}
	in terms of a $C^1$ function $h : \set{X} \subset \R^n \to \R$. If $h$ is a control barrier function on $\set{X}$ and for all $\x \in \partial \set{C}$ it holds that $\partial h(\x) / \partial \x \neq 0$, then any Lipschitz continuous controller $\vec{k}(\x) \in \set{U}_{\mathrm{CBF}}(\x)$ for the system in~\autoref{eq:nonlinear_affine_control} renders the set $\set{C}$ positively invariant.
\end{theorem}
 
Intuitively, any controller $\vec{k}(\x) \in \set{U}_{\mathrm{CBF}}$ guarantees that the Lie derivative of $h$ is nonnegative on the boundary of $\set{C}$. Since $h$ is only nonnegative on $\set{C}$, the set is positively invariant under the assumption of locally Lipschitz closed-loop dynamics.  

Any arbitrary given feedback controller might not satisfy the condition on the Lie derivative from~\autoref{eq:cbf_lie_derivative} of the control barrier function $h$. \cite{freeman_kototovic, Ames2014} introduce a point-wise min-norm controller that minimally modifies the given controller to guarantee positive control invariance. As the CBF condition on positive invariance 
is affine in the control input $\u$, we can define a quadratic program (QP) to find the minimum adjustment to a given controller $\vec{k}(\x)$ that renders the controlled system positively invariant:
\begin{subequations}
	\label{eqn:cbf_qp}
	\begin{align}
	\u^*(\x) = \underset{\u \in \set{U}}{\text{argmin}} & \quad \frac{1}{2} \lVert \u - \vec{k}(\x) \rVert_2^2 \label{eqn:cbf_qp_cost} \\ \text{s.t.} & \quad L_\vec{f} h(\x) + L_\vec{g} h(\x) \u \geq - \gamma(h(\x))\,. \label{eqn:cbf_constraint}
	\end{align}
\end{subequations}
Intuitively, the optimized $\u^*(\x)$ guarantees that the set $\set{C}$ is positively invariant and stays as close as possible to the original controller $\vec{k}(\x)$, where closeness is specified with respect to a chosen distance measure (e.g., the Euclidean norm in~\autoref{eqn:cbf_qp_cost}). Using this approach, we can apply any controller to the system, and the safety filter in \autoref{eqn:cbf_qp} guarantees positive invariance of the system on the safe set $\set{C}$.

Consider the true system as a disturbed nominal system as in~\autoref{eq:disturbed_system},
where $\vec{d}(\x, \u) \in \mathbb{D} \subset \R^n$ is the model mismatch between the nominal and the true system. The Lie derivative of the disturbed dynamics is given by
\begin{equation}
    \label{eq:projected_disturbed_dynamics}
    \frac{\partial h(\x)}{\partial \x} \xdot = L_{\fhat} h(\x) + L_{\ghat} h(\x) \u + \delta (\x, \u) = L_{\fhat_{\mathrm{cl}}} h(\x) + \delta(\x, \u) \,,
\end{equation}
where $\delta (\x, \u) = \partial h(\x) / \partial \x \: \vec{d}(\vec{x}, \vec{u}) \in \R$ is the error in the Lie derivative. We assume that $\lvert \delta(\x, \u) \rvert \leq \rho$ for all $(\vec{x}, \vec{u}) \in \set{X} \times \set{U}$. For a state-feedback controller $\u=\vec{k}(\x)$, the Lie derivative of the disturbed closed-loop system is given by $L_{\f_{\mathrm{cl}}} h(\x) := L_{\fhat_{\mathrm{cl}}} h(\x) + \delta(\x, \u)$. The notion of projection-to-state safety (PSSf) determines safety for uncertain systems. We use a similar definition compared to~\cite{taylor2020b}; however, we directly use the control barrier function candidate $h(\x)$ as the projection:
\begin{definition}[Projection-to-state safety~\citep{taylor2020b}]
The disturbed control system in \\\autoref{eq:disturbed_system} is projection-to-state safe (PSSf) on $\set{C}_{\bar{h}} = \{\x \in \set{X} \subset \R^n \: \vert \: \bar{h}(\x) \geq 0 \}$ if there exists $\rho > 0$
such that the safe set $\set{C} \supset \set{C}_{\bar{h}}$ is positively control invariant for all $\delta(\x, \u)$ satisfying $\lvert \delta(\x, \u) \rvert \leq \rho$ for all $(\vec{x}, \vec{u}) \in \set{X} \times \set{U}$,
where $\bar{h}(\x) = h(\x) - \rho$ with the control barrier function $h$ as the projection and $\rho$ being the bound of the projected disturbances $\delta$ in~\autoref{eq:projected_disturbed_dynamics}. 
\end{definition}
Under the assumption that the set $\set{C}_{\bar{h}}$ is non-empty, we have the following lemma:
\begin{lemma}
\label{cor:pssf}
Let $\set{C}_{\bar{h}} \subset \R^n$ be the 0-superlevel set of a $C^1$ function $\bar{h} : \R^n \to \R$ and for all $\x \in \partial \set{C}_{\bar{h}}$ it holds that $\partial \bar{h}(\x) / \partial \x \neq 0$. If there exists a Lipschitz continuous state-feedback controller $\vec{k} : \R^n \to \R^m$ such that $L_\vec{f} \bar{h}(\x) + L_\vec{g} \bar{h}(\x) \vec{k}(\x) \geq - \gamma (\bar{h}(\x))$,
and there exists $\rho > 0$ such that the projected disturbances $\delta(\x, \vec{k}(\x))$ in~\autoref{eq:projected_disturbed_dynamics} satisfy $\lvert \delta(\x, \vec{k}(\x)) \rvert \leq \rho$ for all $(\x, \vec{k}(\x)) \in \set{X} \times \set{U}$ and $\rho \leq - \gamma (- \rho )$, then there exists $\gamma_{PSSf} \in \set{K}_{\infty, e}$  that renders the disturbed system in~\autoref{eq:disturbed_system} PSSf on $\set{C}_{\bar{h}}$.  
\end{lemma}
\begin{proof}
Let $\gamma_{PSSf}(h(\x)) = \gamma (h(\x) - \rho ) - \gamma (- \rho)$, then $\gamma_{PSSf} \in \set{K}_{\infty, e}$ since $\gamma_{PSSf}(0) = 0$ and 
\begin{subequations}
    \begin{align}
        L_{\f_{\mathrm{cl}}} (\bar{h}(\x) + \rho ) 
        &= L_{\fhat_{\mathrm{cl}}} \bar{h}(\x) + \delta(\x, \vec{k}(\x)) \\
        &\geq L_{\fhat_{\mathrm{cl}}} \bar{h}(\x) - \rho \label{eq:proof_LHS}\\
        &\geq - \gamma( \bar{h}(\x) ) - \rho \\
         &\stackrel{\mathrm{using}~\rho \leq - \gamma (- \rho )}{\geq}  - \gamma( \bar{h}(\x) ) + \gamma (- \rho) \\
        &= - \gamma_{PSSf} (\bar{h} (\x) + \rho ) \label{eq:proof_RHS} \,.
    \end{align}
\end{subequations}
Then, the disturbed system in~\autoref{eq:disturbed_system} is positively control invariant on $\set{C}$ and PSSf on $\set{C}_{\bar{h}}$. 
\end{proof}
In the case that $\rho > - \gamma (- \rho )$, only the greater 0-superlevel set of $h(\x) + h_0$ can be rendered positively invariant instead of $\set{C}$, where $h_0 = \gamma^{-1}(-\rho) +\rho $.

\subsection{Bayesian Linear Regression}
In this section, we briefly describe model learning using Bayesian linear regression (BLR), see \cite{Bishop07} for details. Our goal is to learn the Lie derivative residual as a linear combination of basis functions with additive Gaussian noise using a set of data points. 
Given a dataset $\mathcal{D} = \left\{ \vec{X}, \vec{y} \right\}$ of independent identically distributed data points, where the input set is $\vec{X} = \left(\x_1, \dots, \x_N \right)$ with $\x_i \in \R^n$ and $N > 0$ and the target set is $\vec{y} = \left(y_1, \dots, y_N \right)$ with $y_i \in \R$.
We assume that the Lie derivative residual can be modeled as $\vec{w}^\intercal \phi(\x)$,
where $\vec{w} \in \R^M$ with $M > 0$ are the weights and $\phi(\x) = \begin{pmatrix} \phi_1(\x), \dots, \phi_M(\x) \end{pmatrix}^\intercal$ defines the set of basis functions. 
We want to determine the weights such that the BLR model closely represents the dataset $\set{D}$.
We assume a Gaussian prior distriubtion over the weights $\vec{w}$ as $p(\vec{w}) = \set{N} (\mu_0, \Sigma_0)$,
with mean $\mu_0$ and covariance $\Sigma_0$.
Let $\Phi \in \R^{N \times M}$ be the design matrix with entries $\phi_{i, j} = \phi_j(\x_i)$ and $\sigma^2$ is the noise variance. 
Then, the posterior distribution is given by $    p(\vec{w} | \set{D} ) = \set{N} (\mu_N, \Sigma_N) \,,
$ where $\mu_N = \Sigma_N \left(\Sigma_0^{-1} \mu_0 + 1 /\sigma^2 \Phi^\intercal \vec{y} \right)$ and 
$\Sigma_N^{-1} = \Sigma_0^{-1} + 1/ \sigma^2 \Phi^\intercal \Phi \,.$
For sequentially arriving data, the previous posterior can be used as the prior for the following data point. 
The predictive distribution is given by $p ( y | \x, \set{D}) = \set{N} (\mu_{\mathrm{pred}}, \sigma^2_{\mathrm{pred}}) \,,
$ where $\mu_{\mathrm{pred}} = \mu_N^\intercal \phi(\x)$ and $
        \sigma^2_{\mathrm{pred}} = \phi(\x)^\intercal \Sigma_N \phi(\x) + 1 / \sigma^2$.

\section{Safely Learning Control Barrier Conditions}
In this section, we present our barrier Bayesian linear regression (BBLR) approach that safely learns uncertain control barrier conditions online.
By construction, we guarantee that \textit{(i)} the Lie derivative residual learned by the BBLR is $0$ on the boundary of the safe set $\partial \set{C}$ and \textit{(ii)} there exists an appropriate extended class $\set{K}_{\infty}$ function to lower bound the Lie derivative in the CBF condition. 

\subsection{Exploiting Control Barrier Condition Structure for Safe Online Learning}

The model mismatch between the true and nominal systems directly impacts the Lie derivative in the CBF condition. 
The Lie derivative residual can be written as $\delta (\x, \u) = \alpha (\x) + \beta (\x)^\intercal \u$, where $\alpha (\x) = \partial h(\x) / \partial \x \: \vec{a}(\x)$ and $\beta (\x) = \partial h(\x) / \partial \x \: \vec{b}(\x)$.
We aim to safely learn an estimate $\hat{\delta}(\x, \u) = \hat{\alpha} (\x) + \hat{\beta} (\x)^\intercal \u$ of the Lie derivative residual $\delta(\x, \u)$ from online collected data.
Our proposed online learned safety filter augments an arbitrary Lipschitz continuous controller $\vec{k}(\x)$ using a QP (similar to~\autoref{eqn:cbf_qp}) to guarantee that the true system stays inside the safe set $\set{C}$, while also expanding the set of safe control inputs $\set{U}_{CBF}(\x)$ for the true system. 
We  exploit the structure of the CBF condition by choosing the set of basis functions as compositions of design functions $\phi_j$ and the control barrier function $h$. By construction, this will guarantee that the learned Lie derivative residual will have a value greater than or equal to zero on the boundary $\partial \set{C}$. In contrast, arbitrary function approximators typically do not satisfy this condition.

Since $h(\x)$ is $C^1$ and $\set{C}$ is compact, 
$h$ attains a minimum $ h_{\mathrm{min}} = 0$ and maximum $h_{\mathrm{max}} = \max_{\x \in \set{C}} h(\x)$ over $\set{C}$, 
and the codomain of $h : \set{C} \to \set{H}$ with $\set{H} = \left[ 0, h_{\mathrm{max}} \right]$ is compact. 
Let $\set{F}$ be the set of functions that are lower bounded by a negative class $\set{K}_{\infty, e}$ function on $\set{H}$ with $\set{F} = \left\{ \phi : \set{H} \to \R \:\vert\: \exists \gamma \in \set{K}_{\infty, e},~ \forall r \in \set{H},~ \phi(r) \geq - \gamma(r) \right\}.$ Then, we have the following results:

\begin{lemma}
\label{lem:class_K_lower_bound}
    The set $\set{F}$ is closed under addition. 
\end{lemma}
\begin{proof}
    Let $\phi_j \in \set{F}$ for $j \in \{1, \dots, N \}$ be a finite collection of functions. 
    By definition there exist $\gamma_j$ such that for all $j \in \{1, \dots, N \},$  $\phi_j(r) \geq - \gamma_j(r)$. This implies that $\sum_{j = 1}^N \phi_j(r) \geq - \sum_{j = 1}^N \gamma_j(r) =: - \gamma(r)$, where we used the result from~\cite{Kellett2014}, that a finite sum of class $\set{K}_{\infty, e}$ functions is also of class $\set{K}_{\infty, e}$.  
\end{proof}

\begin{lemma}
\label{lem:learned_cbf_condition}
    The CBF condition for the estimated Lie derivative  $L_{\fhat_{\mathrm{cl}}} \bar{h}(\x) + \hat{\alpha} (\x) + \hat{\beta} (\x)^\intercal \u \geq - \gamma(h(\x))$ is satisfied, if there exist $\gamma_0, \gamma_{\delta} \in \set{K}_{\infty, e}$ such that $L_{\fhat_{\mathrm{cl}}} \bar{h}(\x) \geq - \gamma_0(h(\x))$ and $\hat{\alpha} (\x) + \hat{\beta} (\x)^\intercal \u \geq - \gamma_{\delta}(h(\x))$. 
\end{lemma}
\begin{proof}
    Using the result from~\citep{Kellett2014} on the summation of a finite number of class $\mathcal{K}_{\infty, e}$ functions, we have $L_{\fhat_{\mathrm{cl}}} \bar{h}(\x) + \hat{\alpha} (\x) + \hat{\beta} (\x)^\intercal \u \geq -  \gamma_0(h(\x)) - \gamma_{\delta}(h(\x)) =: - \gamma(h(\x))$. Therefore, the CBF condition in~\autoref{eq:cbf_lie_derivative} is satisfied. 
\end{proof}
The main result of our paper combines the previous Lemmas to guarantee safe control barrier conditions using linear combinations of basis functions $\phi_j \circ h$. The idea is to guarantee that any learned Lie derivative residual $\hat{\delta}$ is lower bounded by an extended class $\set{K}_\infty$ function. 
\begin{theorem}
\label{theo:safety_condition}
Let $L_{\fhat_{\mathrm{cl}}} \bar{h}(\x) + \hat{\alpha} (\x) + \hat{\beta} (\x)^\intercal \u$ be the estimate of the true Lie derivative $ L_{\f_{\mathrm{cl}}} \bar{h}(\x)$ and $\u \in \set{U}$, where $\set{U}$ is compact, and $\hat{\alpha}$ and $\hat{\beta}$ are defined as 
\begin{equation}
\label{eq:alpha_beta_choices}
    \hat{\alpha}(\x) = \sum_{j = 1}^N \lambda_j \phi_j (h (\x))\,,~ \hat{\beta}(\x) = \begin{pmatrix} \sum_{j = 1}^N \mu_{1, j} \phi_j (h (\x)) & \dots & \sum_{j = 1}^N \mu_{m, j} \phi_j (h (\x))\end{pmatrix}^\intercal \,
\end{equation}         
with nonnegative weights $\lambda_j, \mu_{i, j} \in \R_{\geq 0}$, $\phi_j \in \set{F}$. 
If there exists a $\gamma_0 \in \set{K}_{\infty, e}$ such that $L_{\fhat_{\mathrm{cl}}} h(\x) \geq - \gamma_0(h(\x))$, 
then $L_{\fhat_{\mathrm{cl}}} \bar{h}(\x) + \hat{\alpha} (\x) + \hat{\beta} (\x)^\intercal \u \geq - \gamma(h(\x))$ with $\gamma\in \set{K}_{\infty, e}$.
\end{theorem}

\begin{proof}
From Lemma~\ref{lem:learned_cbf_condition}, since $L_{\fhat_{\mathrm{cl}}} h(\x) \geq - \gamma_0(h(\x))$ is satisfied by assumption, we need to show that $\hat{\alpha} (\x) + \hat{\beta} (\x)^\intercal \u \geq - \gamma_{\delta}(h(\x))$ to satisfy the CBF condition. \autoref{eq:alpha_beta_choices} for all $\x \in \set{C}_{\bar{h}}$ yields:
\begin{subequations}
    \begin{align}
        \hat{\alpha} (\x) + \hat{\beta} (\x)^\intercal \u &= \sum_{j = 1}^N \lambda_j \phi_j (h (\x)) + \sum_{i = 1}^m \sum_{j = 1}^N \mu_{i, j} u_i \phi_j (h (\x)) \\
        &\stackrel{\mathrm{using}~\phi_j \in \set{F}}{\geq} - \sum_{j = 1}^N \lambda_j \gamma_j (h (\x)) - \sum_{i = 1}^m \sum_{j = 1}^N \mu_{i, j} \lvert u_i \rvert \gamma_j (h (\x)) \\
         &\stackrel{\mathrm{using~Lemma}~\ref{lem:class_K_lower_bound}}{=:} - \gamma_{\delta}(h(\x)) \,,
    \end{align}
\end{subequations}
where $\gamma_{\delta} \in \set{K}_{\infty, e}$.
\end{proof}
\autoref{theo:safety_condition} shows that any learned Lie derivative residual estimate that is a sum of nonnegatively weighted function compositions $\phi_j \circ h$ with $\phi_j \in \set{F}$ is lower bounded by a negative class $\mathcal{K}_{\infty, e}$ function. 
This guarantees that the sum of the nominal Lie derivative $L_{\fhat_{\mathrm{cl}}} h(\x)$ and the learned Lie derivative residual $(\hat{\alpha} (\x) + \hat{\beta} (\x)^\intercal \u)$ satisfies the CBF condition as well.

\subsection{Locally Learning Control Barrier Conditions with Online Bayesian Linear Regression}

A common assumption in previous work is that the learned Lie derivative satisfies the CBF condition. 
Additionally, the learned Lie derivative did not account for the uncertainty in the numerical Lie derivative estimation. 
We explicitly guarantee the satisfaction of the CBF condition, by leveraging the structure of our presented Lie derivative residual from the previous section in an online framework. We consider the uncertainty in the numerical Lie derivative estimates through Bayesian linear regression, which yields the uncertainty in our predictions.

Since the true system is unknown, we cannot determine the Lie derivative of the true system. Therefore, we estimate the Lie derivative residual locally from recent measurements summarized in a dataset $\mathcal{D}_{t_0, t_1} = \left\{ \vec{X}_{t_0, t_1}, \vec{U}_{t_0, t_1}, \vec{H}_{t_0, t_1} \right\}$, where $\vec{X}_{t_0, t_1} = \left(\x_{t_0}, \dots, \x_{t}, \dots, \x_{t_1} \right)$ with $\x_t = \x(t \Delta_t)$, $\vec{U}_{t_0, t_1} = \left(\u_{t_0}, \dots, \u_{t_1} \right)$, and $\vec{H}_{t_0, t_1} = \left(h(\x_{t_0}), \dots, h(\x_{t_1}) \right)$ where $t_1 \geq t \geq t_0 \geq 0$, $\Delta_t > 0$ is the sampling time, 
$t_1 = k - 1$, where $k$ is the current discretized time step, $t_0 = t_1 - T$, 
and $T$ is the number of past time steps we consider. Then, using the forward difference method we can determine an estimate of the true Lie derivative: $ L_{f_{\mathrm{cl}}} h (\x) \approx (h(\x_{t +1}) - h(\x_t)) / \Delta_t$.
\begin{remark}
The Lie derivative estimation error can be bounded following~\citep{nocedal2017numerical} with $c \Delta_t / 2 := \sigma_{\mathrm{diff}}$,
where $c = \sup_{\x \in \set{C}} d^2 h(\x) / d t^2$.
In practice, the second or any higher time derivative of $h$ cannot be determined. Instead, $\Delta_t$ needs to be chosen sufficiently small and $c$ needs to be chosen conservatively to account for the numerical differentiation errors. 
\end{remark}
Our goal is to use BLR to determine the weights for the basis functions from these local datasets. 
We assume a Gaussian prior over the weights for BLR. This violates the condition of nonnegative weights from~\autoref{theo:safety_condition}. To guarantee that every basis function still satisfies the conditions, we reduce the set of possible basis functions to functions for which $\phi_j \in \set{F}$ and $- \phi_j \in \set{F}$ on the nonnegative domain $\set{H}$, where $\set{F}$ is the class of functions that are lower bounded by negative class $\mathcal{K}_{\infty, e}$ functions. This restriction requires for all $\x \in \set{C}\,,$ $\phi_j(h(\x)) = 0$. This yields safety, however, this excludes positive Lie derivative residuals on the boundary. The class of locally Lipschitz functions on the set $\set{H}$ for which $\phi_j(0) = 0$ satisfies this restriction. The following corollary guarantees safety on safe set $\set{C}$ with the more restricted set of basis functions:
\begin{corollary}
\label{cor:positive_weights}
    Let $\hat{\delta}(\x, \u) = \vec{w}^\intercal \Phi(\x, \u)$, with $\vec{w} \in \R^{N (m + 1)}$ and 
    $\Phi(\x, \u) = \begin{pmatrix} \Phi_{\x}(\x)^\intercal \,,  \Phi_{\u}(\x, \u)^\intercal \end{pmatrix}^\intercal$,
    where $\Phi_{\x}(\x) = \begin{pmatrix} \phi_1  (h(\x)),  \dots,  \phi_N (h(\x)) \end{pmatrix}^\intercal$ and $\Phi_{\u}(\x, \u) = \begin{pmatrix} u_1 \Phi_{\x}(\x)^\intercal \dots  u_m \Phi_{\x}(\x)^\intercal \end{pmatrix}^\intercal$
    and $\phi_j$ is such that $\phi_j \in \set{F}$ and $- \phi_j \in \set{F}$. 
    Then, $L_{\fhat_{\mathrm{cl}}} \bar{h}(\x) + \hat{\alpha} (\x) + \hat{\beta} (\x)^\intercal \u \geq - \gamma(h(\x))$
    is satisfied, if there exists $\gamma_0 \in \set{K}_{\infty, e}$ such that $L_{\fhat_{\mathrm{cl}}} \bar{h}(\x) \geq - \gamma_0(h(\x))$.
\end{corollary}
\begin{proof}
     The estimated Lie derivative residual $\hat{\delta}(\x, \u)$ satisfies
\begin{subequations}
    \begin{align}
         \hat{\delta}(\x, \u) & = \sum_{j = 1}^{N} w_j \phi_j (h (\x)) + \sum_{i = 1}^m \sum_{j = 1}^N w_{iN + j} u_i \phi_j (h (\x)) \,,  \\
        & \geq - \sum_{j = 1}^{N} \lvert w_j \rvert \gamma_j (h (\x)) - \sum_{i = 1}^m \sum_{j = 1}^N \lvert w_{iN + j} \rvert \lVert \u \rVert_{\infty} \gamma_j (h (\x)) \,,  \\
        & =: - \gamma_{\hat{\delta}} (h(\x)) \,.
    \end{align}
\end{subequations}
The rest of the proof follows similar arguments as the proof for~\autoref{theo:safety_condition}.
\end{proof}
The restricted set of basis functions satisfies the safe application of BLR, which yields the predicted mean $\mu_k$ and covariance $\Sigma_k$. The estimated Lie derivative residual is $\hat{\delta}_k(\x, \u) = \mu_k^\intercal \Phi(\x, \u)$ and the covariance of the Lie derivative residual is given by $\sigma_{k}(\x, \u) = \Phi(\x, \u)^\intercal \Sigma_k \Phi(\x, \u) + \frac{1}{\sigma_{\mathrm{diff}}^2}$. 
We assume that the resulting probabilistic model is well calibrated, such that there exists $s \in \N$ so that $\delta(\x, \u) \in \hat{\delta}_k(\x, \u) \pm s \sigma_k(\x, \u)$ with high probability. 
Then, we can use our Lie derivative estimate to render an arbitrary Lipschitz continuous controller $\vec{k}(\x)$ safe using the following QP:     
\begin{align}
	\raisetag{30pt}
	\label{eqn:learning_cbf_qp}
	\begin{split}
	\u_{\mathrm{bblr}}^*(\x) = \underset{\u \in \set{U}}{\text{argmin}} & \quad \frac{1}{2} \lVert \u - \vec{k}(\x) \rVert_2^2 \\ \text{s.t.} & \quad L_{\fhat} \bar{h}(\x) + L_{\ghat} \bar{h}(\x) \u + \hat{\delta}(\x, \u) - \rho \geq - \gamma_{PSSf} (\bar{h} (\x) + \rho ) - \gamma_{\hat{\delta}} (h(\x))\,, 
    \end{split}
\end{align}
where the CBF condition in~\autoref{eqn:learning_cbf_qp} follows by combining the results from Corollary~\ref{cor:positive_weights}, and Lemma~\ref{cor:pssf} with~\autoref{eq:proof_LHS} for the left hand side, and~\autoref{eq:proof_RHS} for the right hand side of the constraint. 
Finally, we account for the uncertainty between the nominal Lie derivative and the true Lie derivative, which is captured by the constant $\rho$, and the uncertainty between the learned Lie derivative and true Lie derivative, which is captured by $\sigma_{k}(\x, \u)$. 
Depending on these uncertainties, we blend the safe control inputs from a nominal filter $\u^*_{\mathrm{nominal}}(\x)$, which is a QP filter subject to the CBF condition as in~Lemma~\ref{cor:pssf}, and our proposed filter, which uses $\u^*_{\mathrm{bblr}}(\x)$ from~\autoref{eqn:learning_cbf_qp}. This yields $\u^* = (1 - r) \u^*_{\mathrm{nominal}}(\x) + r \u^*_{\mathrm{bblr}}(\x)$,
where $r = q \left(\frac{\rho}{\rho + s \sigma_{k}(\x, \u)} \right)$ with a monotonically increasing function $q$ bounded by $1$, and $s \in \N$, which determines the number of considered standard deviations. 
In contrast to previous work~\citep{taylor2020b}, which blended between unsafe and safe control inputs during learning, our approach blends two safe control inputs and guarantees that the system will not leave $\set{C}$ during learning. 
When the uncertainty of the learned Lie derivative residual is larger than $\rho$, the system returns to the smaller safe set $\set{C}_{\bar{h}}$ instead, due to the set's attractiveness~\citep{ames2019a}, until the uncertainty is reduced again.

\section{Simulation Results}
\label{sec:experiments}
We show the performance of our proposed BBLR on the example of an inverted pendulum:
\begin{equation}
    \dot{x}_1  = x_2\,,~
		\dot{x}_2  = - g / l \:\sin (x_1) + 1/ (m l^2) \:u \,,
\end{equation}
where $\begin{pmatrix} x_1, x_2\end{pmatrix}^\intercal= \begin{pmatrix} \theta & \dot{\theta} \end{pmatrix}^\intercal$ with $\theta$ being the angle of the pendulum from the downward position and $\dot{\theta}$ being the angular velocity, $g$ is the gravitational constant, $l$ is the length and $m$ is the mass of the pendulum, and $u$ is the input torque. The parameters of the true and nominal systems are $(l=1.0~\text{m}, m=1.0~\text{kg})$ and $(l=1.0~\text{m}, m=0.96~\text{kg})$, respectively. We conservatively estimate $\rho = 0.5$.
We are given the CBF $h(\theta, \dot{\theta})  = 1 - ( \theta/\theta_{\mathrm{max}} )^2 - ( \dot{\theta}/\dot{\theta}_{\mathrm{max}} )^2 $, where $\theta_{\mathrm{max}} = \pi~\text{rad}$ and $\dot{\theta}_{\mathrm{max}} = \SI{5.0}{\radian/\sec}$, which yields an ellipsoidal safe set that prevents the pendulum from swinging beyond the upright position. 
We design a controller based on the nominal system with the goal of swinging up the pendulum (see~\autoref{fig:plots}~(a)): $u = (-18.8\: (\theta - \pi) - 6.1\:\dot{\theta}  +5)~\text{Nm}$. 
For BBLR we consider the set of monomials from degree 1 to 5 composited with $h(\theta,  \dot{\theta})$ as the basis and learn from the last $\SI{0.2}{\second}$ of data. For the first $\SI{0.2}{\second}$, we use the nominal filter and afterwards we apply the blended safe inputs as described in the previous section. 

\autoref{fig:plots}~(b) shows the successful application of the learned CBF filter using BBLR. The unfiltered controller, leaves the safe set, while the filtered controllers keep the system inside the safe set. Our learned safety filter is able to keep the system at the boundary of the safe set, whereas the nominal filter conservatively keeps the system inside a conservative nominal safe set. Finally, \autoref{fig:plots}~(c) demonstrates the satisfaction of the CBF condition through BBLR as the Lie derivative evaluated for the true system reduces to $0$ at the boundary. In the case of unfiltered control inputs, the system achieves a  negative Lie derivative on the boundary resulting in the violation of safety. The nominal filter restricts the  Lie derivative to 0 at the boundary of the more conservative nominal safe set. 

\begin{figure}
    \centering
    \subfigure[]
    {\raisebox{8mm}{
        \includegraphics[scale=.2]{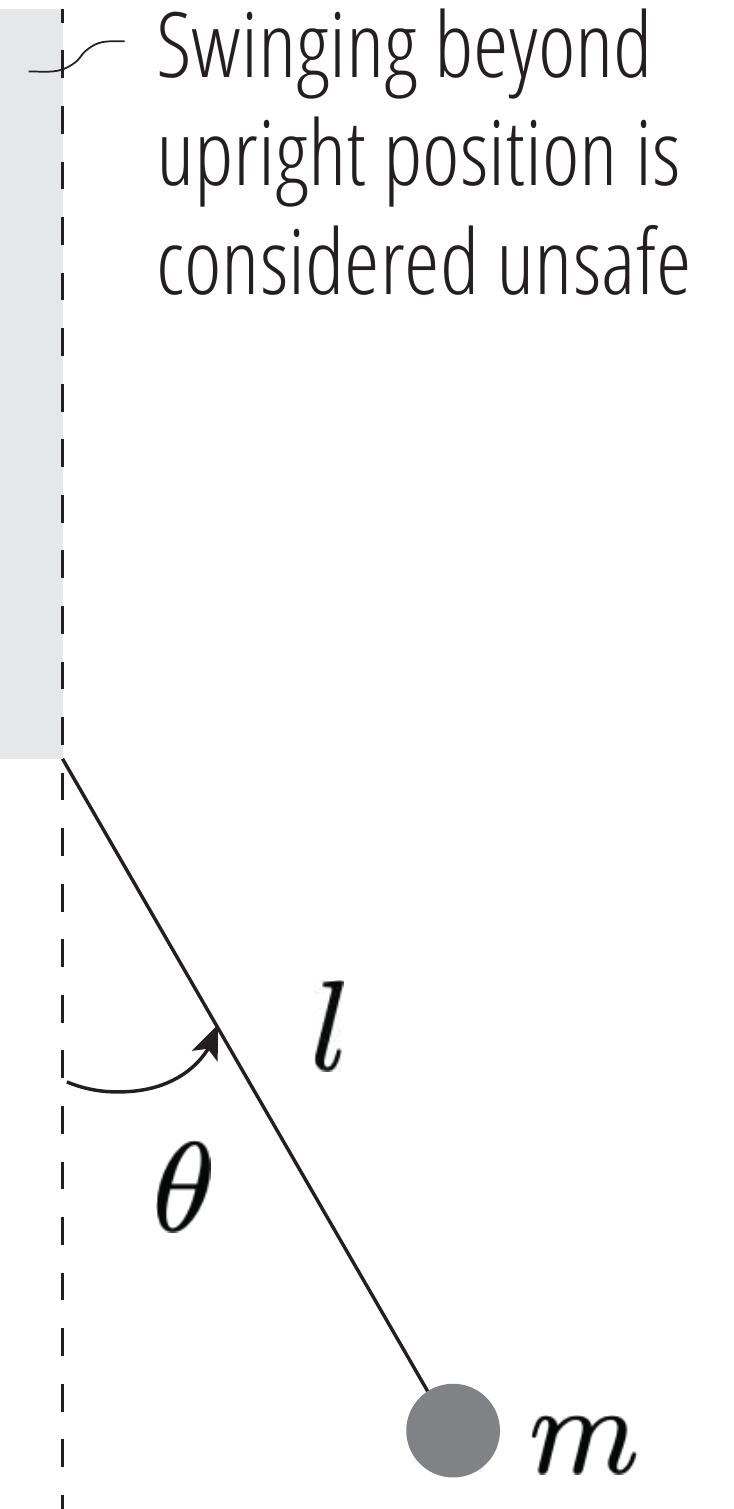}
        }
    }
    ~
    \subfigure[]
    {
        \includegraphics[scale=.4, trim={0.0cm 0cm 1.6cm 1.3cm},clip]{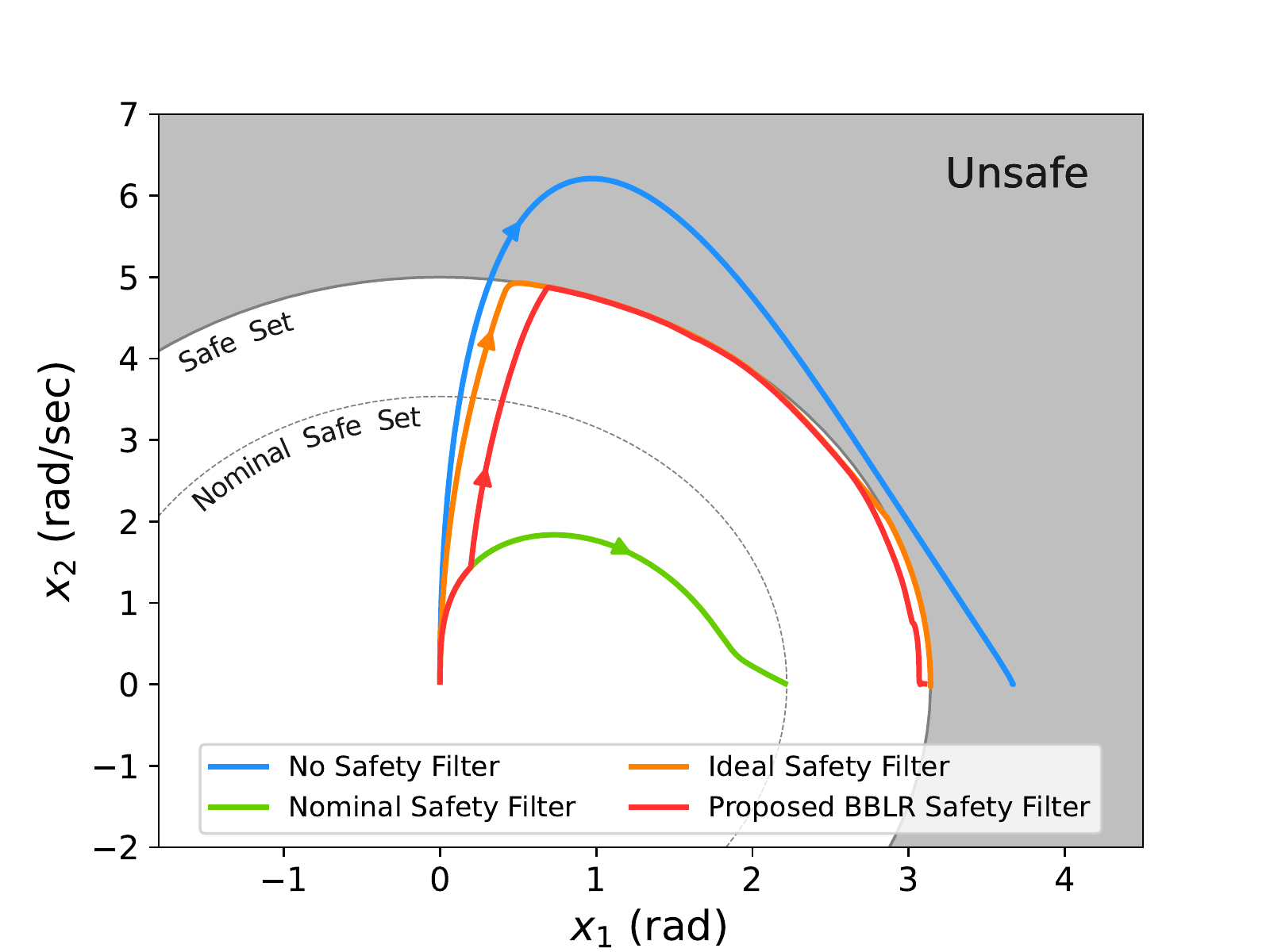}
    }~
    \subfigure[]
    {
        \includegraphics[scale=.4, trim={0.0cm 0cm 1.6cm 1.3cm},clip]{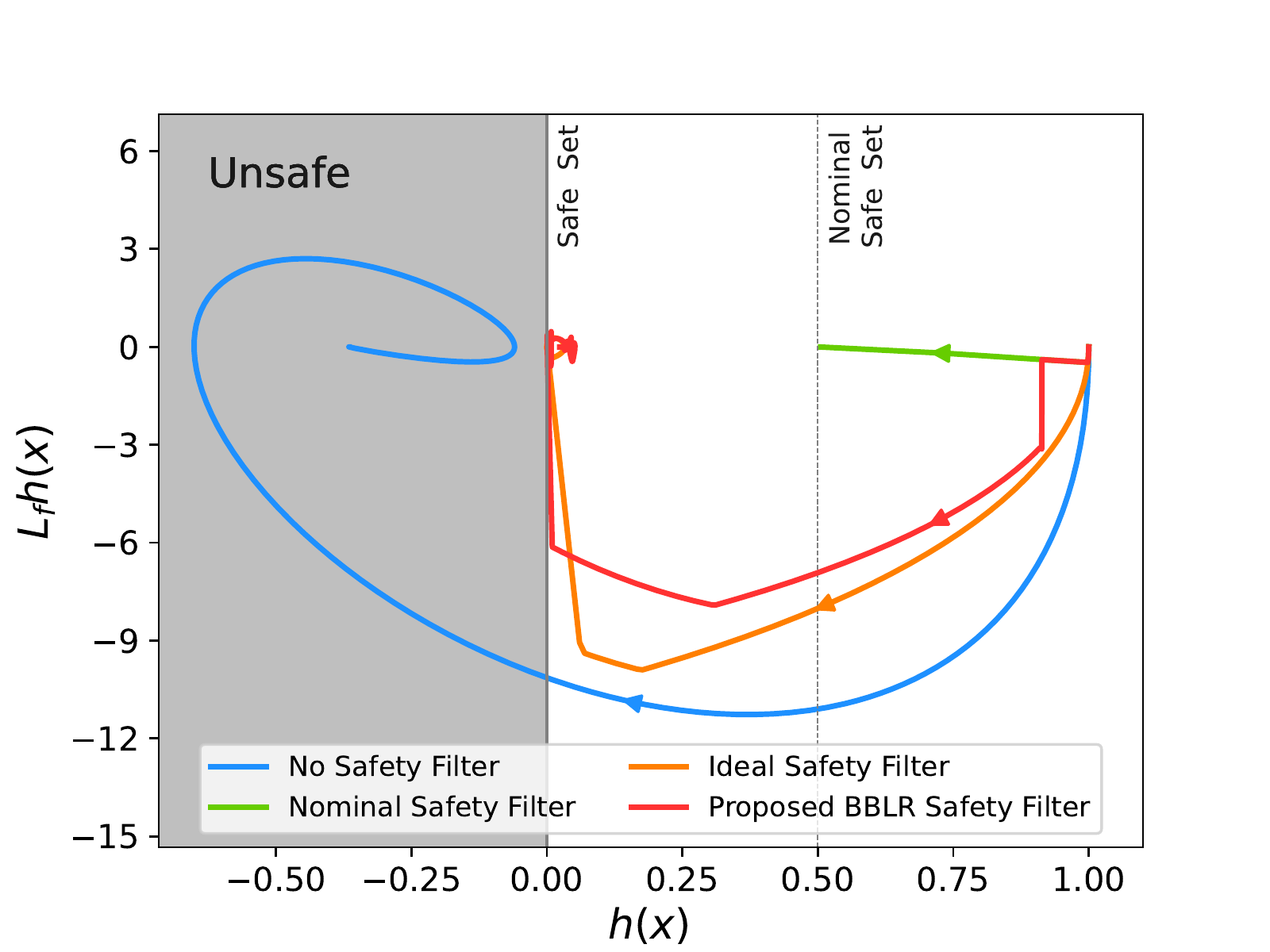}
    }
    \caption
    {
        (a) The schematic of the pendulum;
        (b) the state trajectories demonstrate that our proposed BBLR safe filter achieves similar performance as compared to the filter using the true CBF condition, while also guaranteeing safety of the safe set; 
        (c) the Lie derivative over the CBF evaluated for the true system confirms that our learned Lie derivative is guaranteed to be 0 on the boundary of the safe set. 
    }
    \label{fig:plots}
\end{figure}

\section{Conclusions}
We proposed the BBLR approach for safely learning control barrier conditions online. Our approach can be applied to partially unknown nonlinear control-affine systems with bounded uncertainty. By leveraging a special set of basis functions that are lower-bounded by negative class $\mathcal{K}_{\infty, e}$ functions, we were able to guarantee the satisfaction of the CBF condition while learning the Lie derivative residual using BLR. We exploit the predictive uncertainty provided by the BLR to systematically account for the errors in the online learning scheme. We demonstrated the improved performance and safety of our approach on a simulated two-dimensional pendulum. 

\bibliography{references}

\end{document}